\documentclass[sigconf]{acmart}
\AtBeginDocument{%
  }

\setcopyright{acmlicensed}

\copyrightyear{2026}
\acmYear{2026}
\setcopyright{cc}
\setcctype{by}
\acmConference[KDD '26]{Proceedings of the 32nd ACM SIGKDD Conference on Knowledge Discovery and Data Mining V.2}{August 09--13, 2026}{Jeju Island, Republic of Korea}
\acmBooktitle{Proceedings of the 32nd ACM SIGKDD Conference on Knowledge Discovery and Data Mining V.2 (KDD '26), August 09--13, 2026, Jeju Island, Republic of Korea}
\acmDOI{10.1145/3770855.3817877}
\acmISBN{979-8-4007-2259-2/2026/08}

\usepackage{booktabs}
\usepackage{tabularx}
\usepackage{colortbl}
\usepackage{arydshln} 
\usepackage{multirow}
\usepackage{graphicx}
\usepackage{subcaption}
\usepackage{amsmath}
\usepackage{balance}

\begin{document}

\title{Physically-Constrained Mamba-SDE for Remaining Useful Life Prediction under Irregular Observations}


\author{Deyu Zhuang}
\affiliation{%
  \institution{Nanjing University of Aeronautics and Astronautics}
  \city{Nanjing}
  \country{China}}
\email{deyu.zhuang@nuaa.edu.cn}

\author{Peiliang Gong}
\authornote{Corresponding Author}
% \affiliation{%
%   \institution{Singapore University of Technology and Design}
%   \country{Singapore}
% }
\affiliation{%
  \institution{Nanyang Technological University}
  \country{Singapore}
}
\email{plgong@outlook.com}

\author{Yang Shao}
\affiliation{%
 \institution{Nanjing University of Aeronautics and Astronautics}
  \city{Nanjing}
  \country{China}}
\email{shaoyang@nuaa.edu.cn}

\author{Liyuan Shu}
\affiliation{%
 \institution{Nanjing University of Aeronautics and Astronautics}
  \city{Nanjing}
  \country{China}}
\email{lyshu1106@nuaa.edu.cn}

\author{Qi Zhu}
\affiliation{%
 \institution{Nanjing University of Aeronautics and Astronautics}
  \city{Nanjing}
  \country{China}}
\email{zhuqi@nuaa.edu.cn}

\author{Xiaoli Li}
\affiliation{%
  \institution{Singapore University of Technology and Design}
  \country{Singapore}}
\email{xiaoli_li@sutd.edu.sg}

\author{Daoqiang Zhang}
% \authornote{Corresponding Author}
\authornotemark[1]
\affiliation{%
 \institution{Nanjing University of Aeronautics and Astronautics}
  \city{Nanjing}
  \country{China}}
\email{dqzhang@nuaa.edu.cn}

\renewcommand{\shortauthors}{Zhuang et al.}

\begin{abstract}
Accurate Remaining Useful Life prediction is critical for industrial predictive maintenance. However, real-world deployment is challenging due to the irregular nature of sensor observations, characterized by asynchronous sampling, burst missingness, and temporal jitter. Compounding this issue, purely data-driven models often generate physically implausible degradation trajectories that violate the irreversible nature of damage accumulation.
To address this, we propose PC-MambaSDE, a unified continuous-time framework for robust RUL prediction under irregular observations. Specifically, we design a Mask-Aware Continuous Mamba Encoder that explicitly leverages observation masks to extract context-rich control signals. Furthermore, we introduce a Physics-Guided Latent SDE with parametrically rectified hybrid drift, superimposing a global physical bias to enforce monotonic degradation even amid severe observation gaps. Additionally, we formulate RUL prediction as a boundary value problem via a Terminal Degradation Penalty, which decouples a Health Index dimension and applies a penalty loss to guide trajectories toward the failure state. Theoretically, we prove that our variational objective is mathematically equivalent to minimizing the KL divergence via Girsanov’s theorem, and we guarantee the global asymptotic stability of the learned dynamics through Lyapunov analysis.
To enable rigorous evaluation, we develop a Hybrid Irregularity Generation Scheme that simulates realistic industrial imperfections. Extensive experiments on public benchmarks demonstrate that PC-MambaSDE significantly outperforms state-of-the-art methods, particularly under extreme observation scarcity, validating the efficacy of embedding physical priors into continuous-time latent dynamics. Our code is available at \url{https://github.com/KylinToeFish/PC-MambaSDE}.
\end{abstract}


\begin{CCSXML}
<ccs2012>
   <concept>
       <concept_id>10010405.10010432.10010439</concept_id>
       <concept_desc>Applied computing~Engineering</concept_desc>
       <concept_significance>500</concept_significance>
       </concept>
   <concept>
       <concept_id>10010147.10010257.10010293.10010294</concept_id>
       <concept_desc>Computing methodologies~Neural networks</concept_desc>
       <concept_significance>500</concept_significance>
       </concept>
 </ccs2012>
\end{CCSXML}

\ccsdesc[500]{Applied computing~Engineering}
\ccsdesc[300]{Computing methodologies~Neural networks}

\keywords{Remaining Useful Life Prediction, Sparse and Irregular Data, Stochastic Differential Equations, Mamba, Physical Constraints}


\maketitle

\section{Introduction}

\begin{figure}[t]
  \centering
  \includegraphics[width=1.0\linewidth, trim=4cm 6cm 3cm 6cm, clip]{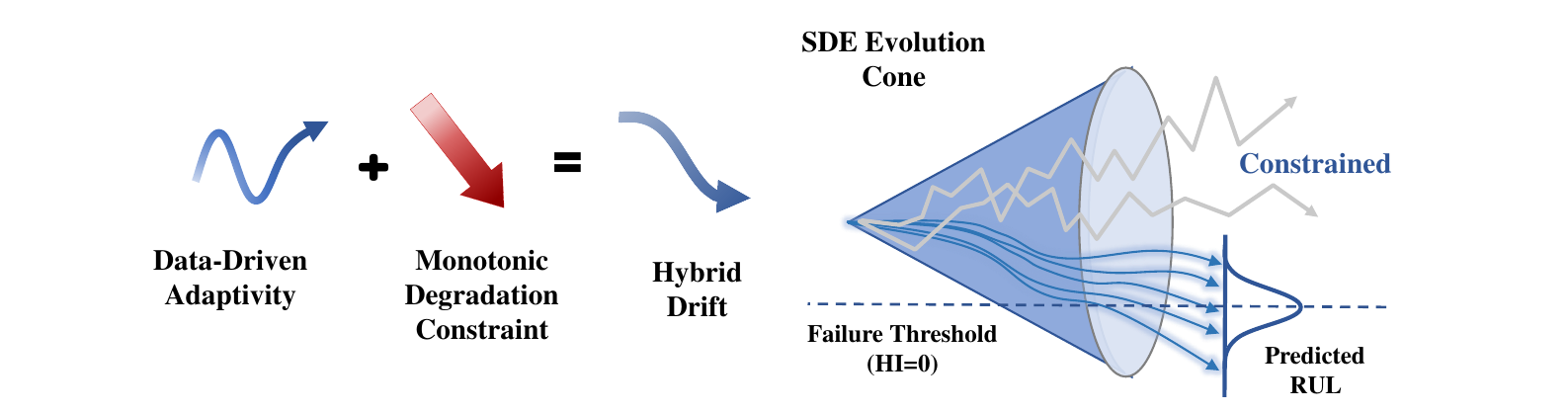}
  \caption{Generative Mechanism of PC-MambaSDE via Parametrically Rectified Hybrid Drift.}
  \label{fig:drift}
\end{figure}

Predictive Maintenance (PdM) has emerged as a cornerstone of modern industrial reliability, enabling proactive management of high-value assets such as aircraft engines and manufacturing systems \cite{ref1, ref2}. Central to PdM is the estimation of Remaining Useful Life (RUL), a problem typically cast as a multivariate time series regression task where the objective is to map historical sensor trajectories to a time-to-failure target \cite{ref3, ref4}. Driven by the availability of large-scale benchmarks like C-MAPSS, data-driven approaches have evolved rapidly. Early works utilized Convolutional Neural Networks (CNNs) to capture local temporal features \cite{ref5, ref6}, while Recurrent Neural Networks (RNNs), including LSTMs and GRUs, became the standard for modeling long-term temporal dependencies \cite{ref7, ref8, ref9}. More recently, attention-based mechanisms have been explored to further enhance interpretability and performance in dense data regimes \cite{ref10, ref11}.

However, real-world industrial deployment is often hampered by sparse and irregular time series, resulting from sensor malfunctions, discrete sampling, or network latency \cite{ref12, ref13}. Traditional sequential models, which assume fixed sampling intervals, exhibit substantial performance degradation under such conditions  \cite{ref14}. Conventional remedies employ imputation or interpolation techniques—such as cubic splines, Gaussian processes, or functional principal component analysis (FPCA)—to realign data onto regular temporal grids \cite{ref15, ref16, ref17}. However, as demonstrated by Wang et al. \cite{ref18}, such two-stage ``impute-then-predict'' pipelines often introduce significant bias and amplify uncertainty, particularly when degradation signals are highly scarce. Similarly, Cheng et al. \cite{ref19} proposed a parameterized static regression framework that mitigates data scarcity through historical posterior rectification. While these methods effectively address certain impacts of sparsity, they remain discrete or static in nature, failing to capture the continuous-time evolution intrinsic to physical degradation processes.

To overcome the inherent limitations of discrete-time architectures, the machine learning community has increasingly adopted Neural Differential Equations. By formulating latent state evolution as a continuous-time process, models such as Neural Controlled Differential Equations (Neural CDEs) \cite{ref20} and Latent Stochastic Differential Equations (Latent SDEs) \cite{ref21, ref22} naturally accommodate sporadic observations without ad-hoc interpolation. This paradigm has enabled sophisticated architectures including GRU-ODE-Bayes \cite{ref23} and Neural Continuous-Discrete State Space Models (SSMs) \cite{ref24}, offering theoretically principled frameworks for irregular temporal dynamics. While continuous-time modeling has emerged as a dominant approach for general irregular data processing, its application to RUL prediction remains relatively unexplored and confronts domain-specific challenges.

Standard Latent SDEs operate as ``black-box'' generative models lacking the physical constraints inherent to mechanical degradation. Specifically, they fail to guarantee the monotonicity of the degradation process or ensure that the trajectory converges strictly to a failure state at end-of-life \cite{ref25}. Furthermore, purely data-driven neural drift functions exhibit spurious oscillations or unphysical stagnation in data-sparse regions, violating the fundamental irreversibility of damage accumulation. Although Park et al. \cite{ref26} recently proposed ACSSM to approximate posterior path measures via amortized control, ensuring stochastic trajectories rigorously adhere to physical endpoint constraints without computationally prohibitive approximations remains an open problem. Additionally, effective latent dynamics necessitate robust context encoding from irregular observations. While modern SSMs such as Mamba \cite{ref27} and Structured State Space sequence model (S4) \cite{ref28} excel at sequence modeling, their standard implementations do not explicitly exploit the informative missingness patterns characteristic of PdM data \cite{ref29}, limiting their capacity to generate precise control signals for downstream differential equation solvers.

To address these critical gaps, we propose \textbf{\underline{P}}hysically-\textbf{\underline{C}}onstrained \textbf{\underline{Mamba-SDE}} (\textbf{PC-MambaSDE}), a unified continuous-time framework for robust RUL prediction under sparse and irregular conditions. 
Specifically, we design a Mask-Aware Continuous Mamba Encoder that explicitly fuses observation masks with sensor values within the selective state space mechanism, extracting continuous, context-rich control signals to precisely guide SDE evolution. As illustrated in Figure \ref{fig:drift}, we develop a Parametrically Rectified Hybrid Drift mechanism that superimposes a global, learnable physical bias onto local neural vector fields, guaranteeing monotonic degradation even in observation-sparse regions. Furthermore, we introduce a Terminal Degradation Penalty that explicitly decouples a "Health Index" (HI) dimension within the latent space. By penalizing terminal deviations, we guide stochastic processes to converge physically to failure states, bridging data-driven generation and physical degradation laws.

We rigorously evaluate our method on widely-adopted benchmarks using a novel Hybrid Irregularity Generation Scheme (HIGS) that systematically simulates realistic industrial imperfections including asynchronous dropout, structural burst missingness, temporal jitter, and state-dependent noise. Extensive experiments demonstrate that PC-MambaSDE significantly outperforms state-of-the-art baselines, particularly in highly scarce data scenarios, validating the effectiveness of integrating physical constraints with continuous-time stochastic modeling.
The main contributions of this paper can be summarized as follows:
\begin{itemize}
    \item We propose PC-MambaSDE, a unified continuous-time framework designed for RUL prediction under sparse and irregular observations. By integrating state space models with stochastic differential equations, our approach effectively bridges the gap between discrete, data-driven representation learning and continuous physical degradation laws.
    \item We design a Mask-Aware Continuous Mamba Encoder to extract continuous, context-rich control signals by explicitly fusing observation masks with time-series data. Furthermore, we propose a Physics-Guided Latent SDE with a Parametrically Rectified Hybrid Drift and a Terminal Degradation Penalty, which enforces monotonic degradation trends and strictly constrains trajectories toward the failure boundary.
    \item We formulate a multidimensional irregularity benchmark that simulates realistic industrial imperfections, such as asynchronous dropout and burst missingness. Extensive experiments on C-MAPSS and N-CMAPSS datasets demonstrate that our model achieves superior robustness and accuracy compared to state-of-the-art baselines, particularly in scenarios with extreme data scarcity.
    \item We provide a rigorous theoretical analysis of the proposed framework. We prove that minimizing the control energy is mathematically equivalent to maximizing the Evidence Lower Bound via Girsanov’s theorem, and we strictly establish the global asymptotic stability of the basis-decomposed drift using Lyapunov theory. 
\end{itemize}

\section{Related Work}

\subsection{Data-Driven RUL Prediction}
In the past decade, data-driven Prognostics and Health Management has shifted from traditional feature engineering to deep learning architectures. CNNs were among the first to effectively extract local temporal features from multivariate sensor data. Babu et al. \cite{ref5} introduced a deep CNN for RUL estimation, and subsequent works like Li et al. \cite{ref6} and Yang et al. \cite{ref30} further improved performance by using time-windowed approaches and double-convolution structures to capture multi-scale patterns. To address the sequential nature of degradation, RNN architectures, particularly LSTMs, became the dominant paradigm. Zheng et al. \cite{ref7} and Heimes \cite{ref3} demonstrated the superiority of LSTMs in capturing long-term dependencies in the C-MAPSS benchmark.
More recently, attention mechanisms have been introduced to weight critical time steps. Chen et al. \cite{ref10} proposed an attention-based deep learning framework, while Qin et al. \cite{ref11} and Xu et al. \cite{ref31} integrated feature-temporal attention to better focus on sensitive features. However, these methods typically assume that sensor data is dense and regularly sampled. When applied to sparse or irregular industrial data, they often require preprocessing steps like interpolation, which can introduce significant bias \cite{ref18}. Unlike these discrete-time approaches, our PC-MambaSDE operates in continuous time, inherently accommodating irregularity without reconstruction.

\subsection{Modeling Irregular Time Series}
Handling irregular sampling is a fundamental challenge in time series analysis. Classic approaches rely on imputation methods, such as spline interpolation \cite{ref15} or FPCA \cite{ref16, ref17}. However, Wang et al. \cite{ref18} argued in their work on Sparse MFMLP that such two-stage ``impute-then-learn'' strategies are suboptimal for scarce data. Similarly, Cheng et al. \cite{ref19} recently proposed a parameterized static regression framework to bypass interpolation by rectifying predictions with historical estimates. While these methods mitigate sparsity, they remain static or discrete, failing to model the underlying continuous dynamics of the system.
A more rigorous solution lies in Neural Differential Equations. Neural ODEs \cite{ref32} and Neural CDEs \cite{ref20} treat the hidden state as a continuous function of time, naturally handling arbitrary time intervals. For stochastic systems, Latent SDEs \cite{ref21, ref22} extend this by incorporating Brownian motion to model uncertainty. Advanced frameworks like GRU-ODE-Bayes \cite{ref23} and Neural Continuous-Discrete SSMs \cite{ref24} have shown great performance in medical and climate domains. Recently, Park et al. \cite{ref26} introduced the ACSSM, employing amortized control and Doob’s $h$-transform to approximate posterior path measures for irregular time series. While ACSSM excels in sequence modeling, it lacks the monotonic constraints essential for fatigue degradation. We address this by integrating a physics-constrained terminal penalty into SDE dynamics.

\subsection{State Space Models and Generative Models}
State Space Models have seen a resurgence as efficient alternatives to Transformers for long sequence modeling. Gu et al. \cite{ref28} introduced the S4, and subsequently Mamba \cite{ref27}, which utilizes a selective scanning mechanism to achieve linear complexity. These models are theoretically grounded in continuous systems, making them suitable encoders for differential equations. Nevertheless, standard Mamba implementations are designed for continuous data streams and lack explicit mechanisms to handle the informative missingness characteristic of PdM data \cite{ref29}. In terms of generation, while Generative Adversarial Networks have been explored for time-series regeneration \cite{ref33}, they often lack the mathematical tractability of SDEs. By proposing a Mask-Aware Continuous Mamba Encoder, we combine the efficiency of modern SSMs with the irregularity-robustness of SDEs, creating a unified framework that is both computationally efficient and physically consistent.

\section{Hybrid Irregularity Generation Scheme}
\label{sec:higs}

To evaluate the robustness of RUL prediction models under realistic industrial conditions, we propose the Hybrid Irregularity Generation Scheme. Real-world sensor networks rarely provide ideal data; they suffer from packet loss, clock desynchronization, and noise. HIGS simulates these complex stochasticities through a four-stage injection process, establishing a benchmark that reflects the true entropy of industrial environments.

Let $\mathbf{X} \in \mathbb{R}^{L \times D}$ be the pristine sensor measurements and $\mathbf{t} \in \mathbb{R}^L$ be the normalized timestamps. The HIGS transform $\mathcal{T}_{\text{HIGS}}(\mathbf{X}, \mathbf{t})$ applies the following perturbations:

\paragraph{1. Asynchronous Sensor Dropout ($p_{drop}$):}
To simulate random packet loss in IoT transmission, we apply an independent Bernoulli mask to each sensor channel. Let $m_{i,j} \sim \text{Bernoulli}(1 - p_{\text{drop}})$ be the availability mask for the $j$-th sensor at step $i$. The observed value becomes $\tilde{x}_{i,j} = x_{i,j} \odot m_{i,j}$.

\paragraph{2. Structural Burst Missingness ($\lambda_{burst}, \mu_{len}$):}
Systematic failures often cause continuous segments of data loss. We model this via a Poisson process where we sample the number of burst events $k \sim \text{Poisson}(\lambda_{\text{burst}})$ and their durations $l \sim \mathcal{N}(\mu_{\text{len}}, \sigma_{\text{len}})$. For each event starting at random index $t_{\text{start}}$, we enforce $\mathbf{m}_{t_{\text{start}} : t_{\text{start}}+l} = \mathbf{0}$, creating structural voids in the timeline.

\paragraph{3. Temporal Jitter ($\sigma_{jitter}$):}
Industrial loggers often suffer from clock synchronization errors. We perturb the timestamps using Gaussian noise: $\tilde{t}_i = t_i + \epsilon_i$, where $\epsilon_i \sim \mathcal{N}(0, \sigma_{\text{jitter}}^2)$. Crucially, to preserve physical causality, we enforce a monotonicity constraint $\tilde{t}_i = \max(\tilde{t}_i, \tilde{t}_{i-1} + \delta_{\min})$.

\begin{figure*}[t]
  \centering
  \includegraphics[width=0.92 \textwidth, trim=7cm 0cm 6cm 0cm, clip]{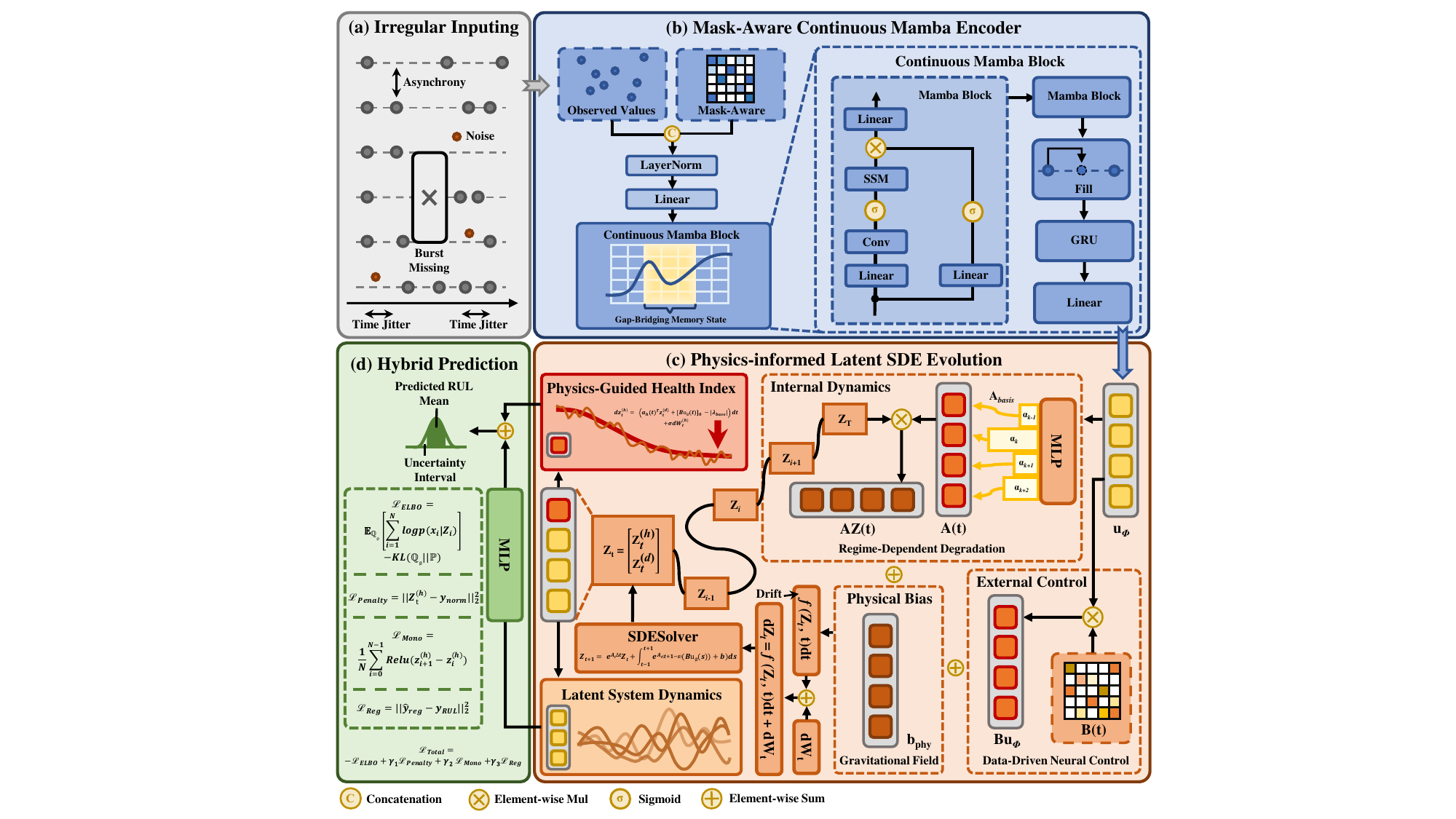} 
  
  \caption{The framework of PC-MambaSDE.(a) Irregular Input Processing: Handles industrial sensor data with asynchronous sampling, burst missingness, and temporal jitter.(b) Mask-Aware Continuous Mamba Encoder: Extracts continuous control signals $u_{\phi}(t)$ from sparse sequences by explicitly fusing validity masks.(c) Physics-Guided Latent SDE: Governs state evolution via a hybrid drift, superimposing a global physical bias $b_{phy}$ onto local neural vector fields.(d) Hybrid Prediction: Jointly optimizes RUL distributions using ELBO, a terminal penalty, and monotonicity regularization.}
  \label{fig:framework}
\end{figure*}

\paragraph{4. State-Dependent Noise Injection ($\alpha_{noise}$):}
Recognizing that sensor noise often correlates with system degradation, we introduce a state-dependent noise term. The noise variance scales dynamically with the current HI:
\begin{equation}
\mathbf{x}_{\text{final}} = \tilde{\mathbf{x}} + \xi, \quad \xi \sim \mathcal{N}\left(0, \sigma_{\text{base}}^2 \left(1 + \alpha_{\text{noise}} \cdot (1 - \text{HI}_t)\right)\right)
\end{equation}
where $\alpha_{\text{noise}}$ controls the sensitivity of noise to the degradation severity. This creates a challenging ``signal-to-noise ratio decay'' scenario typical of late-stage machinery failure.

\section{Methodology}

\subsection{Problem Formulation: Stochastic
Degradation with Irregular Observations}
Let the health state of a system be represented by a continuous-time stochastic process $\mathbf{Z}_t \in \mathbb{R}^d$ defined over a finite horizon $t \in [0, T]$. The dynamics are governed by a generic It\^o SDE:
\begin{equation}
\mathrm{d}\mathbf{Z}_t = \mathbf{f}(\mathbf{Z}_t, t) \mathrm{d}t + \sigma \mathrm{d}\mathbf{W}_t, \quad \mathbf{Z}_0 \sim p(\mathbf{Z}_0)
\end{equation}
where $\mathbf{f}: \mathbb{R}^d \times \mathbb{R}^+ \to \mathbb{R}^d$ is the drift function governing the degradation trend, and $\mathbf{W}_t$ is a standard Wiener process representing intrinsic stochasticity.

In industrial scenarios, we observe a sequence of discrete, irregularly sampled measurements $\mathcal{D} = \{(\tau_i, \mathbf{x}_{\tau_i}, \mathbf{m}_{\tau_i})\}_{i=0}^N$, where $\tau_i$ are timestamps, $\mathbf{x}_{\tau_i}$ are sensor readings, and $\mathbf{m}_{\tau_i}$ is the observation mask. The observations are emitted from the latent state via a linear map $\mathbf{H}$ subject to Gaussian noise:
\begin{equation}
\mathbf{x}_{\tau_i} = \mathbf{H}\mathbf{Z}_{\tau_i} + \boldsymbol{\epsilon}_i, \quad \boldsymbol{\epsilon}_i \sim \mathcal{N}(\mathbf{0}, \Sigma_{\text{obs}})
\end{equation}
Our goal is to infer the posterior path measure $\mathbb{Q}(\mathbf{Z}_{[0,T]} | \mathcal{D})$ and predict RUL $y_{RUL}$ based on the terminal state $\mathbf{Z}_T$.

\subsection{Overview}
We propose PC-MambaSDE, a unified continuous-time framework that bridges discrete, irregular observations with continuous physical degradation processes. As shown in Figure \ref{fig:framework}, our architecture comprises three synergistic components. First, a Mask-Aware Continuous Mamba Encoder processes irregular inputs—sensor values, observation masks, and temporal intervals—to generate continuous-time control signals that explicitly exploit sparsity patterns. Second, a Physics-Guided Latent SDE governs state evolution via a Parametrically Rectified Hybrid Drift, superimposing a global physical degradation bias onto local neural vector fields to guarantee monotonic degradation while adapting to system-specific dynamics. Third, a Terminal Degradation Penalty reformulates RUL prediction as a boundary value problem, decoupling a HI dimension and enforcing endpoint constraints that compel trajectories to traverse from healthy states to failure boundaries. Together, these components integrate the expressive power of SSMs, the rigor of SDEs, and domain-informed physical constraints. Theoretically, our framework is rigorously grounded: we prove the equivalence of our objective to the ELBO via Girsanov’s theorem and guarantee global stability through Lyapunov analysis, with the detailed derivation provided in \textbf{Appendix~\ref{app:proof}}. We detail each component below.

\subsection{Mask-Aware Mamba Encoder Architecture}

Standard Transformers treat time series as ordered sequences with fixed intervals, often struggling with the asynchronous and sparse nature of industrial sensor data. To address this, we design a hierarchical encoder that seamlessly integrates the efficient long-sequence modeling of Mamba with a robust mechanism for handling data gaps. Unlike quadratic-complexity attention mechanisms, our Mask-Aware Mamba Encoder operates with linear complexity, which we empirically validate in \textbf{Appendix~\ref{app:time}}. The architecture consists of three key stages: Mask-Aware Embedding, Mamba Encoding, and Recurrent Latent Filling.

\subsubsection{Augmented Embedding}
We explicitly couple the sensor readings $\mathbf{x}_{\tau_i}$ with their corresponding validity masks $\mathbf{m}_{\tau_i}$ to prevent the model from confusing ``missing data'' with ``zero values.'' These concatenated features are projected into a high-dimensional latent space $\mathbf{X}_{emb} \in \mathbb{R}^{B \times L \times D_{model}}$ via a learnable linear projection:
\begin{equation}
\mathbf{X}_{emb} = \text{LayerNorm}(\text{Linear}_{in}([\mathbf{x}_{\tau_i} \oplus \mathbf{m}_{\tau_i}]))
\end{equation}
This embedding strategy ensures that the sparsity pattern serves as an explicit feature, initializing the sequence modeling with awareness of observation validity.

\subsubsection{Mamba Sequence Modeling}
We denote the stacked Mamba operation as $\Phi_{\text{Mamba}}$, which processes the augmented embedding to yield the preliminary latent sequence $\mathbf{H}_{seq}$:
\begin{equation}
\mathbf{H}_{seq} = \Phi_{\text{Mamba}}(\mathbf{X}_{emb})
\end{equation}

\subsubsection{Latent Filling and Recurrent Smoothing}
To strictly handle the irregularity and stabilize the control signal for the downstream SDE, we introduce a Latent Forward-Filling mechanism followed by a Recurrent Smoothing layer.

\paragraph{Latent Forward-Filling:} Recognizing that the degradation state of a mechanical system cannot change abruptly during signal loss, we apply the sample-and-hold strategy. For any time step $t$ with missing data, we propagate the latent state from the last valid observation. Let $\tau(t) = \max \{k \le t \mid \text{valid observation exists at } k\}$ be the index of the most recent valid observation. The filled latent state $\mathbf{H}_{fill}$ is defined as:
\begin{equation}
\mathbf{H}_{fill}[t] = \mathbf{H}_{seq}[\tau(t)]
\end{equation}

\paragraph{Recurrent Smoothing:} Finally, to smooth out local discontinuities introduced by the filling operation and to refine the trajectory dynamics, we process the filled sequence through a Gated Recurrent Unit (GRU). This layer acts as a dynamic filter, generating the final continuous control signal $\mathbf{u}_\phi(t)$ required to guide the SDE drift:
\begin{equation}
\mathbf{u}_\phi(t) = \text{GRU}(\mathbf{H}_{fill})
\end{equation}
The resulting $\mathbf{u}_\phi(t)$ encodes a robust, continuous-time representation of the system's degradation state, effectively enabling the subsequent SDE solver to generate monotonic and physics-compliant trajectories even under extreme data sparsity.

\subsection{Physics-Guided Latent SDE Dynamics}

Purely data-driven Neural SDEs are often ``black boxes'' that produce physically implausible trajectories. Furthermore, learning a high-dimensional drift matrix $\mathbf{A}$ is notoriously unstable. To mitigate these issues, we introduce Latent Decoupling to enforce interpretability and Basis Decomposition to guarantee dynamic stability.

\subsubsection{Decoupled State Space}
Unlike standard approaches that entangle all latent dimensions, we impose a semantic structure on the state space. We partition the latent state $\mathbf{Z}_t$ into a scalar HI $z^{(h)}_t$ and a hidden dynamic vector $\mathbf{z}^{(d)}_t$:
\begin{equation}
\mathbf{Z}_t = \begin{bmatrix} z^{(h)}_t \\ \mathbf{z}^{(d)}_t \end{bmatrix} \in \mathbb{R}^{1 + (d-1)}
\end{equation}
The dimension $z^{(h)}_t$ is explicitly designated to track the cumulative damage, serving as a virtual sensor for the unobservable degradation status. The remaining dimensions $\mathbf{z}^{(d)}_t$ capture complex, non-monotonic system variations.

\subsubsection{Basis-Decomposed Drift Construction}
Directly learning the drift matrix $\mathbf{A}$ often leads to eigenvalues with positive real parts, causing the SDE integration to diverge over long time horizons. To ensure Lyapunov stability, we decompose the drift matrix into a weighted sum of pre-conditioned stable basis matrices. The encoder output $\mathbf{u}_\phi$ dynamically predicts the mixing coefficients $\boldsymbol{\alpha}_k$:
\begin{equation}
\boldsymbol{\alpha} = \text{Softmax}(\text{Linear}_{coeff}(\mathbf{u}_\phi)) \in \mathbb{R}^K, \quad \sum_{k=1}^K \alpha_k = 1
\end{equation}
The basis matrices $\mathbf{A}_{basis}^{(k)}$ are parameterized to be strictly negative-definite using an exponential constraint. This guarantees that the real parts of the eigenvalues remain negative, forcing the system to return to equilibrium in the absence of external forcing:
\begin{equation}
\mathbf{A}_{basis}^{(k)} = -(\exp(\mathbf{D}_{basis}^{(k)}) + \epsilon \mathbf{I}), \quad \epsilon=10^{-6}
\end{equation}
The time-varying neural drift matrix $\mathbf{A}_{neural}(t)$ is then synthesized as a convex combination of these bases, allowing the model to switch between different dynamic regimes smoothly:
\begin{equation}
\mathbf{A}_{neural}(t) = \sum_{k=1}^{K} \alpha_k(t) \cdot \mathbf{A}_{basis}^{(k)}
\end{equation}

To guarantee that this construction prevents the stochastic process from diverging in observation-sparse regions, we establish the following theorem:

\begin{theorem}[Asymptotic Stability of Convex Combinations]
If a set of matrices $\{A_k\}_{k=1}^K$ share a common quadratic Lyapunov function and are Hurwitz (stable), then their convex combination $A(t) = \sum_{k=1}^K \alpha_k(t) A_k$ is also Hurwitz.
\end{theorem}

\begin{proof}
We parameterize each basis matrix as $A_k = -(\exp(D_k) + \epsilon I)$. Let $v$ be any non-zero vector. The quadratic form is:
\begin{equation}
    v^T A_k v = -v^T \exp(D_k) v - \epsilon v^T v
\end{equation}
Since $\exp(D_k)$ is positive semi-definite (matrix exponential of real parameters) and $\epsilon > 0$, we have:
\begin{equation}
    v^T A_k v \le -\epsilon \|v\|^2 < 0
\end{equation}
Thus, each basis matrix $A_k$ is strictly negative definite. Now consider the time-varying system drift $A(t) = \sum \alpha_k(t) A_k$, where $\alpha_k(t) \ge 0$ and $\sum \alpha_k = 1$ (Output of Softmax). The derivative of the Lyapunov candidate $V(Z) = Z^T Z$ along the trajectory is:
\begin{equation}
    \dot{V}(Z) = Z^T (A(t)^T + A(t)) Z = \sum_{k=1}^K \alpha_k(t) Z^T (A_k^T + A_k) Z
\end{equation}
Since $Z^T A_k Z < 0$ for all $k$, the weighted sum is strictly negative:
\begin{equation}
    \dot{V}(Z) < 0, \quad \forall Z \neq 0
\end{equation}
This implies global asymptotic stability for the homogeneous part of the SDE. The system will inherently revert to the equilibrium (or the path dictated by $b_{phy}$) in the absence of control $u_\phi$, preventing explosion in sparse data regions.
\end{proof}

\subsubsection{Parametric Drift Rectification}
Standard Neural SDEs are agnostic to the ``arrow of time'' in degradation physics. To enforce the irreversibility of damage, we propose a Parametric Drift Rectification mechanism. We formulate the posterior drift function $\mathbf{f}_\phi(\mathbf{Z}_t, t)$ as the superposition of three distinct forces: the internal system dynamics, the external neural control, and a constant physical field:
\begin{equation}
\mathrm{d}\mathbf{Z}_t = \left( \mathbf{A}_{neural}(t) \mathbf{Z}_t + \mathbf{B} \mathbf{u}_\phi(t) + \mathbf{b}_{\text{phy}} \right) \mathrm{d}t + \sigma \mathrm{d}\mathbf{W}_t
\end{equation}
The term $\mathbf{b}_{\text{phy}}$ serves as a static gravitational field in the latent space. We structurally enforce a negative bias on the HI dimension while leaving the dynamic dimensions unbiased. This is derived from the Arrhenius principle, where degradation naturally accumulates:
\begin{equation}
\mathbf{b}_{\text{phy}} = \left[ -|\lambda_{\text{base}}|, \mathbf{0}^\top \right]^\top, \quad \lambda_{\text{base}} \in \mathbb{R}^+
\end{equation}
By expanding Equation (13) for the HI dimension $z^{(h)}_t$, we can explicitly observe the rectification effect:
\begin{equation}
\mathrm{d}z^{(h)}_t = \left( \mathbf{a}_{h}(t)^\top \mathbf{z}^{(d)}_t + [\mathbf{B}\mathbf{u}_\phi(t)]_0 - |\lambda_{\text{base}}| \right) \mathrm{d}t + \sigma \mathrm{d}W^{(h)}_t
\end{equation}
Here, $-|\lambda_{\text{base}}|$ provides a persistent downward momentum, guaranteeing monotonic degradation even in observation-sparse regions, with the detailed derivation provided in  \textbf{Appendix~\ref{app:proof:Penalty}}. The neural terms $\mathbf{a}_{h}(t)^\top \mathbf{z}^{(d)}_t$ and $[\mathbf{B}\mathbf{u}_\phi(t)]_0$ model local fluctuations and load-dependent variations, but the base degradation rate strictly prevents the model from learning ``self-healing'' behaviors commonly seen in unconstrained regression models.

\subsection{Variational Objective via Girsanov Theorem}

We adopt a training strategy based on the principle of minimal intervention. The model should rely on the known physical prior ($\mathbf{b}_{\text{phy}}$) as much as possible, utilizing the neural control $\mathbf{u}_\phi$ only when necessary to explain data anomalies.

We train the model to maximize the ELBO. Let $\mathbb{P}$ be the measure of the uncontrolled prior process and $\mathbb{Q}_\phi$ be the measure of the controlled posterior.
\begin{equation}
\mathcal{L}_{\text{ELBO}} = \mathbb{E}_{\mathbb{Q}_\phi} \left[ \sum_{i=1}^N \log p(\mathbf{x}_{\tau_i} | \mathbf{Z}_{\tau_i}) \right] - \text{KL}(\mathbb{Q}_\phi || \mathbb{P})
\end{equation}
The KL divergence is analytically derived using the Girsanov Theorem, which quantifies the energy required to steer the diffusion process. It manifests as the quadratic norm of the control signal difference between the posterior and the prior:
\begin{equation}
\text{KL}(\mathbb{Q}_\phi || \mathbb{P}) = \mathbb{E}_{\mathbb{Q}_\phi} \left[ \int_0^T \frac{1}{2\sigma^2} || \mathbf{f}_\phi(\mathbf{Z}_t, t) - \mathbf{f}_{\text{prior}}(\mathbf{Z}_t) ||_2^2 \mathrm{d}t \right]
\end{equation}

Substituting our linear parameterization, minimizing the magnitude of the neural control $\mathbf{u}_\phi$ is mathematically equivalent to minimizing the KL divergence between the posterior and the physical prior, as explicitly derived via Girsanov’s Theorem. To ground this relationship with rigorous theoretical guarantees, we state and derive the following theorem:

\begin{theorem}[Girsanov Theorem for Diffusion Processes]
Let $W_t$ be a Brownian motion under measure $\mathbb{P}$. Let $u(t)$ be an adapted process satisfying Novikov's condition. Define a new measure $\mathbb{Q}$ via the Radon-Nikodym derivative:
\begin{equation}
    \frac{d\mathbb{Q}}{d\mathbb{P}} = \exp\left( \int_{0}^{T} u(t)^T dW_t - \frac{1}{2} \int_{0}^{T} \|u(t)\|^2 dt \right)
\end{equation}
Then, under $\mathbb{Q}$, the process $\tilde{W}_t = W_t - \int_{0}^{t} u(s) ds$ is a standard Brownian motion.
\end{theorem}

\paragraph{Derivation of the Objective:}
Let the Prior SDE (Physical Prior $\mathbb{P}$) be governed by the base degradation physics:
\begin{equation}
    dZ_t = (A(t)Z_t + b_{phy})dt + \sigma dW_t
\end{equation}
Let the Posterior SDE (Neural Control $\mathbb{Q}_\phi$) be governed by the encoded control:
\begin{equation}
    dZ_t = (A(t)Z_t + b_{phy} + B u_\phi(t))dt + \sigma dW_t^{\mathbb{Q}}
\end{equation}
Comparing the drift terms, the excess drift introduced by the neural network is $f_{diff}(t) = B u_\phi(t)$. Assuming constant diffusion $\sigma=I$ for simplicity (or scaling $u$ by $\sigma^{-1}$), the difference in drift is exactly the control signal.

The Kullback-Leibler divergence between the path measures $\mathbb{Q}_\phi$ and $\mathbb{P}$ over $[0, T]$ is given by the expectation of the log Radon-Nikodym derivative:
\begin{align}
    \mathrm{KL}(\mathbb{Q}_\phi || \mathbb{P}) &= \mathbb{E}_{\mathbb{Q}_\phi} \left[ \log \frac{d\mathbb{Q}_\phi}{d\mathbb{P}} \right] \nonumber \\
    &= \mathbb{E}_{\mathbb{Q}_\phi} \left[ \int_{0}^{T} (B u_\phi(t))^T dW_t + \frac{1}{2} \int_{0}^{T} \|B u_\phi(t)\|^2 dt \right]
\end{align}
Since the stochastic integral term $\int_0^T \phi_t dW_t$ is a martingale with zero expectation (assuming $\phi_t$ satisfies the square-integrability condition), where $\phi_t$ represents the integrand. The KL divergence simplifies to the ``Control Energy'':
\begin{equation}
    \mathrm{KL}(\mathbb{Q}_\phi || \mathbb{P}) = \mathbb{E}_{\mathbb{Q}_\phi} \left[ \int_{0}^{T} \frac{1}{2} \|B u_\phi(t)\|^2 dt \right]
\end{equation}

This proves that minimizing the $L_2$-norm of the control signal $u_\phi$ is mathematically equivalent to minimizing the KL divergence between the learned degradation process and the pure physical degradation model.

This effectively regularizes the model to deviate from the physical base degradation rate only when supported by observational evidence:

\begin{equation}
\mathcal{L}_{KL} \approx \mathbb{E}_{\mathbb{Q}_\phi} \left[ \sum_{i=0}^{N-1} \Delta \tau_i \cdot \frac{1}{2} || \mathbf{u}_\phi(\tau_i) ||_2^2 \right]
\end{equation}

\subsection{Hybrid Decoding with Terminal Penalty}

Recognizing that RUL prediction conceptually involves guiding the degradation process toward a failure state, we employ Indirect Supervision to penalize terminal deviations.

We define the RUL target $y_{\text{norm}} \in [0, 1]$. We guide the terminal state to converge toward this target via a Terminal Penalty:
\begin{equation}
\mathcal{L}_{\text{Penalty}} = || z^{(h)}_T - y_{\text{norm}} ||_2^2
\end{equation}
We introduce a Monotonicity Regularizer on the trajectory:
\begin{equation}
\mathcal{L}_{\text{Mono}} = \frac{1}{N} \sum_{i=0}^{N-1} \text{ReLU}\left( z^{(h)}_{\tau_{i+1}} - z^{(h)}_{\tau_{i}} \right)
\end{equation}
Finally, to capture complex non-linear mappings, we include an auxiliary regression head:
\begin{equation}
\hat{y}_{\text{reg}} = \text{MLP}_{dec}(\mathbf{Z}_T), \quad \mathcal{L}_{\text{Reg}} = || \hat{y}_{\text{reg}} - y_{\text{RUL}} ||_2^2
\end{equation}
The total objective function unifies data likelihood, control energy, and physical constraints:
\begin{equation}
\mathcal{L}_{\text{Total}} = -\mathcal{L}_{\text{ELBO}} + \gamma_1 \mathcal{L}_{\text{Penalty}} + \gamma_2 \mathcal{L}_{\text{Mono}} + \gamma_3 \mathcal{L}_{\text{Reg}}
\end{equation}
Substituting the terms, the final optimization problem becomes:
\begin{equation}
\min_\phi \left( \int_0^T \frac{1}{2} ||\mathbf{u}_\phi||^2 \mathrm{d}t - \sum_{i=1}^N \log \mathcal{N}(\mathbf{x}_i | \mathbf{H}\mathbf{Z}_i) + \mathcal{L}_{\text{Phy}} \right)
\end{equation}
where $\mathcal{L}_{\text{Phy}}$ encapsulates the boundary and monotonicity constraints. This formulation unifies data-driven variational inference with physics-informed boundary value problems.

\begin{table*}[t]
\centering
\caption{RMSE Performance Comparison on C-MAPSS under Varying Data Scarcity Levels.}
\label{tab:results}
\renewcommand{\arraystretch}{1.0} 

\begin{tabularx}{0.9\textwidth}{c l *{4}{>{\centering\arraybackslash}X}}
\toprule
\textbf{Scarcity} & \textbf{Method} & \textbf{FD001} & \textbf{FD002} & \textbf{FD003} & \textbf{FD004} \\
\midrule

% --- 50% Group ---
\multirow{8}{*}{50\%} 
& LSTM+GP      & 24.13 $\pm$ 0.35 & 25.09 $\pm$ 0.10 & 25.56 $\pm$ 0.29 & 27.18 $\pm$ 0.84 \\
& LSTM+MPACE   & 19.04 $\pm$ 0.10 & 21.56 $\pm$ 0.59 & 20.84 $\pm$ 0.30 & 23.64 $\pm$ 0.34 \\
& S-MFMLP & 18.65 $\pm$ 0.40 & 20.08 $\pm$ 0.73 & 20.09 $\pm$ 1.30 & 22.22 $\pm$ 0.74 \\
& Latent-ODE   & 22.46 $\pm$ 1.85 & 19.90 $\pm$ 0.85 & 20.71 $\pm$ 0.91 & 21.64 $\pm$ 0.58 \\
& Latent-SDE   & 20.57 $\pm$ 1.01 & 20.28 $\pm$ 0.59 & 21.13 $\pm$ 0.86 & 21.55 $\pm$ 0.47 \\
& PSR          & 18.65 $\pm$ 1.52 & 19.24 $\pm$ 0.41 & 20.54 $\pm$ 0.56 & 22.50 $\pm$ 0.26 \\
& ACSSM        & 18.08 $\pm$ 1.06 & 19.58 $\pm$ 2.12 & 18.72 $\pm$ 1.57 & 20.02 $\pm$ 0.27 \\ 
% \cmidrule(lr){2-6}
\rowcolor{gray!15}
& \textbf{PC-MambaSDE (Ours)} & \textbf{16.68 $\pm$ 0.80} & \textbf{18.55 $\pm$ 0.37} & \textbf{17.90 $\pm$ 0.13} & \textbf{17.11 $\pm$ 0.34} \\
\midrule

% --- 70% Group ---
\multirow{8}{*}{70\%} 
& LSTM+GP      & 31.92 $\pm$ 0.82 & 33.77 $\pm$ 0.10 & 31.79 $\pm$ 0.15 & 30.68 $\pm$ 0.75 \\
& LSTM+MPACE   & 25.31 $\pm$ 2.47 & 27.60 $\pm$ 0.44 & 25.64 $\pm$ 0.34 & 26.21 $\pm$ 0.52 \\
& S-MFMLP & 22.27 $\pm$ 0.55 & 23.79 $\pm$ 0.25 & 22.82 $\pm$ 0.58 & 25.78 $\pm$ 0.33 \\
& Latent-ODE   & 22.85 $\pm$ 0.65 & 20.87 $\pm$ 1.03 & 22.82 $\pm$ 1.78 & 23.69 $\pm$ 0.35 \\
& Latent-SDE   & 21.37 $\pm$ 0.56 & 20.49 $\pm$ 0.23 & 23.21 $\pm$ 0.60 & 23.96 $\pm$ 0.36 \\
& PSR          & 21.24 $\pm$ 0.16 & 21.24 $\pm$ 0.19 & 21.48 $\pm$ 0.43 & 25.01 $\pm$ 0.17 \\
& ACSSM        & 20.15 $\pm$ 0.22 & 19.75 $\pm$ 0.15 & 19.69 $\pm$ 0.55 & 20.90 $\pm$ 0.54 \\
% \cmidrule(lr){2-6}
\rowcolor{gray!15}
& \textbf{PC-MambaSDE (Ours)} & \textbf{18.24 $\pm$ 1.01} & \textbf{18.87 $\pm$ 0.41} & \textbf{18.42 $\pm$ 1.43} & \textbf{19.22 $\pm$ 0.19} \\
\midrule

% --- 90% Group ---
\multirow{8}{*}{90\%} 
& LSTM+GP      & 38.50 $\pm$ 2.94 & 38.78 $\pm$ 0.22 & 35.18 $\pm$ 0.83 & 42.08 $\pm$ 0.47 \\
& LSTM+MPACE   & 31.62 $\pm$ 1.99 & 31.34 $\pm$ 0.29 & 29.14 $\pm$ 0.67 & 33.50 $\pm$ 0.39 \\
& S-MFMLP & 26.13 $\pm$ 4.52 & 30.27 $\pm$ 0.20 & 25.39 $\pm$ 0.76 & 32.39 $\pm$ 0.76 \\
& Latent-ODE   & 24.15 $\pm$ 0.39 & 24.38 $\pm$ 0.39 & 25.96 $\pm$ 1.16 & 25.42 $\pm$ 1.45 \\
& Latent-SDE   & 23.76 $\pm$ 2.40 & 23.79 $\pm$ 0.14 & 25.28 $\pm$ 0.44 & 25.59 $\pm$ 0.12 \\
& PSR          & 28.36 $\pm$ 2.03 & 21.40 $\pm$ 0.10 & 27.67 $\pm$ 0.69 & 32.34 $\pm$ 0.47 \\
& ACSSM        & 22.23 $\pm$ 0.20 & 22.75 $\pm$ 0.15 & 21.24 $\pm$ 1.86 & 24.84 $\pm$ 0.15 \\
% \cmidrule(lr){2-6}
\rowcolor{gray!15}
& \textbf{PC-MambaSDE (Ours)} & \textbf{19.43 $\pm$ 1.63} & \textbf{20.24 $\pm$ 0.22} & \textbf{20.01 $\pm$ 0.19} & \textbf{22.18 $\pm$ 0.63} \\
\bottomrule
\end{tabularx}
\end{table*}

\begin{table*}[t]
\centering
\caption{Performance comparison under different irregularity settings on N-CMAPSS dataset.}
\label{tab:irregularity}

\small 
% 极限压缩列间距
\setlength{\tabcolsep}{1pt} 
\renewcommand{\arraystretch}{1.0} 

\begin{tabularx}{0.9\textwidth}{l *{8}{>{\centering\arraybackslash}X}}
\toprule
\textbf{Setting} & \textbf{LSTM+GP} & \textbf{LSTM+MPACE} & \textbf{S-MFMLP} & \textbf{Latent-ODE} & \textbf{Latent-SDE} & \textbf{PSR} & \textbf{ACSSM} & \textbf{Ours} \\
\midrule
Base & 42.94 $\pm$ 5.11 & 31.48 $\pm$ 2.44 & 20.65 $\pm$ 2.71 & 22.74 $\pm$ 0.95 & 21.13 $\pm$ 0.10 & 9.06 $\pm$ 3.88 & 8.38 $\pm$ 2.07 & \textbf{7.37 $\pm$ 0.92} \\
$\lambda_{burst}$=0.15 & 48.74 $\pm$ 0.40 & 38.64 $\pm$ 1.09 & 30.52 $\pm$ 0.41 & 35.27 $\pm$ 4.01 & 34.46 $\pm$ 3.74 & 21.25 $\pm$ 5.33 & 25.27 $\pm$ 2.60 & \textbf{16.07 $\pm$ 1.78} \\
$\lambda_{burst}$=0.3 & 49.24 $\pm$ 1.36 & 42.56 $\pm$ 2.39 & 38.93 $\pm$ 7.85 & 37.63 $\pm$ 3.23 & 36.85 $\pm$ 8.44 & 23.61 $\pm$ 3.26 & 29.52 $\pm$ 1.44 & \textbf{19.07 $\pm$ 0.49} \\
$\mu_{len}$=10 & 44.44 $\pm$ 3.50 & 34.71 $\pm$ 4.85 & 21.29 $\pm$ 5.83 & 28.55 $\pm$ 0.29 & 27.69 $\pm$ 0.31 & 10.71 $\pm$ 7.69 & 9.04 $\pm$ 3.56 & \textbf{8.08 $\pm$ 1.28} \\
$\mu_{len}$=20 & 47.09 $\pm$ 2.57 & 36.32 $\pm$ 9.79 & 23.25 $\pm$ 2.18 & 31.44 $\pm$ 0.32 & 29.63 $\pm$ 0.33 & 15.04 $\pm$ 2.24 & 12.12 $\pm$ 1.86 & \textbf{8.39 $\pm$ 1.69} \\
$\sigma_{jitter}$=0.25 & 50.87 $\pm$ 5.14 & 35.44 $\pm$ 2.29 & 20.96 $\pm$ 7.82 & 23.86 $\pm$ 1.98 & 23.00 $\pm$ 2.03 & 13.60 $\pm$ 5.41 & 15.89 $\pm$ 1.99 & \textbf{9.60 $\pm$ 0.98} \\
$\sigma_{jitter}$=0.5 & 53.70 $\pm$ 5.38 & 38.20 $\pm$ 2.44 & 21.97 $\pm$ 8.20 & 24.55 $\pm$ 3.04 & 23.62 $\pm$ 2.14 & 14.11 $\pm$ 2.62 & 19.84 $\pm$ 1.52 & \textbf{9.90 $\pm$ 1.21} \\
$\alpha_{noise}$=1.5 & 50.91 $\pm$ 5.08 & 35.57 $\pm$ 1.52 & 21.56 $\pm$ 2.36 & 23.65 $\pm$ 1.47 & 22.39 $\pm$ 2.64 & 15.27 $\pm$ 3.26 & 13.18 $\pm$ 0.94 & \textbf{12.21 $\pm$ 0.33} \\
$\alpha_{noise}$=3 & 53.50 $\pm$ 5.82 & 36.40 $\pm$ 2.54 & 21.95 $\pm$ 4.38 & 24.13 $\pm$ 2.85 & 22.56 $\pm$ 5.35 & 15.25 $\pm$ 4.71 & 14.17 $\pm$ 2.05 & \textbf{12.25 $\pm$ 1.59} \\
\bottomrule
\end{tabularx}
\end{table*}

\section{Experiments}
\subsection{Datasets and Experimental Setup}

\subsubsection{Datasets.} We utilize the C-MAPSS \cite{ref34} for standard RUL scenarios. We use N-CMAPSS \cite{ref2}, preprocessing high-frequency snapshots into cycle-aggregated trajectories via a mean-reduction strategy to retain degradation trends while reducing redundancy. Detailed preprocessing steps are provided in \textbf{Appendix~\ref{app:data}}.

\subsubsection{Baselines.} We compare PC-MambaSDE against a comprehensive suite of baselines categorized into Discrete Sequence Models and Continuous Neural Differential Equations (Neural DEs). The discrete category includes standard LSTM variants (e.g., LSTM+GP, LSTM+MPACE) and specialized sparse modeling approaches such as Sparse MFMLP (S-MFMLP) \cite{ref18} and the Parameterized Static Regression (PSR) \cite{ref19}, which are designed to handle data scarcity through imputation or regression rectification. The continuous category, representing the current state-of-the-art for irregular time series, comprises Latent-ODE \cite{ref32}, Latent-SDE \cite{ref21}, and the amortized control framework ACSSM \cite{ref26}, all of which naturally accommodate non-uniform sampling intervals by modeling the latent state evolution in continuous time.
\subsubsection{Irregularity Injection.} Unlike prior works assuming ideal sampling, to simulate the entropy of industrial IoT environments during inference, we strictly assess robustness under a realistic HIGS as detailed in \textbf{Section~\ref{sec:higs}}.

\subsubsection{Implementation Details.} All experiments were conducted using PyTorch 2.4.1 (Python 3.8) on a single NVIDIA RTX 3060 GPU (CUDA 12.4). We utilized the Adam optimizer with a learning rate of $1 \times 10^{-3}$. The models were trained for up to 50 epochs. To ensure statistical reliability, all reported results represent the average and standard deviation of five independent trials following random initialization.

\subsection{Comparative Experiments}
Tables 1-2 detail the results across various cross-domain scenarios and datasets, categorized into discrete-time baselines, continuous-time Neural ODE/SDE methods, and PC-MambaSDE. For a consistent evaluation, all models are tested under varying levels of data scarcity and irregularity settings to establish robustness bounds.

\subsubsection{Quantitative Results on C-MAPSS Dataset}
To ensure a rigorous assessment of industrial robustness, all experiments are conducted under HIGS. Table~\ref{tab:results} reports performance on the C-MAPSS benchmarks, where we simulate data scarcity by varying the asynchronous sensor dropout rate ($p_{drop} \in \{0.5, 0.7, 0.9\}$) while maintaining a base level of structural irregularity ($\mu_{len}=5$) and temporal jitter ($\sigma_{jitter}=0.1$). PC-MambaSDE consistently achieves state-of-the-art performance, particularly in the most challenging scenarios. For instance, on the FD004 subset with 90\% missingness, our method reduces RMSE to 22.18, significantly outperforming the strongest continuous baseline, ACSSM (24.84), and discrete approaches like Sparse MFMLP (32.39), which struggle to bridge large temporal gaps. Additionally, we provide a direct comparison using the specific experimental settings proposed by Cheng et al. [19] in \textbf{Appendix~\ref{app:psr}}, where our model demonstrates superior accuracy even under their static regression evaluation protocols.

\subsubsection{Quantitative Results on N-CMAPSS Dataset}
To assess scalability and resilience against specific irregularity types, Table~\ref{tab:irregularity} details performance on the N-CMAPSS dataset. Our framework achieves a state-of-the-art average RMSE of 11.44, surpassing both the static PSR (15.32) and purely data-driven differential equations. Notably, under extreme structural burst settings ($\mu_{len}=20$), where continuous segments of data are lost, PC-MambaSDE remains stable ($RMSE=8.39$) while Latent-SDE degrades significantly ($RMSE=29.63$). This validates that the Physics-Guided Drift effectively prevents the "hallucination" of invalid states during long observation gaps. Furthermore, the decoupling of the HI proves critical in high-noise environments ($\alpha=3$), filtering out non-monotonic perturbations that mislead standard SDEs.

\subsection{Ablation Study}
To understand the contribution of each architectural component in PC-MambaSDE, we conduct ablation experiments focusing on the Mask-Aware Mamba Encoder, the Parametric Physical Bias, and the Terminal Degradation Penalty. As shown in Figure \ref{fig:ablation}, removing any single module leads to a noticeable decline in performance, indicating that the full combination is critical for robustness. Specifically, while the Mamba encoder proves indispensable for capturing extended temporal dependencies in long sequences, the physical constraints are equally vital for preventing unphysical "self-healing" hallucinations, particularly in sparse data regimes. These findings show that the data-driven encoder and the physics-informed constraints capture complementary aspects of the problem, specifically temporal feature extraction and physical consistency, both of which are essential for precise RUL prediction under irregular observations.

\begin{figure}[t]
  \centering
  \includegraphics[width=0.7\linewidth, trim=6cm 1.5cm 4cm 1cm, clip]{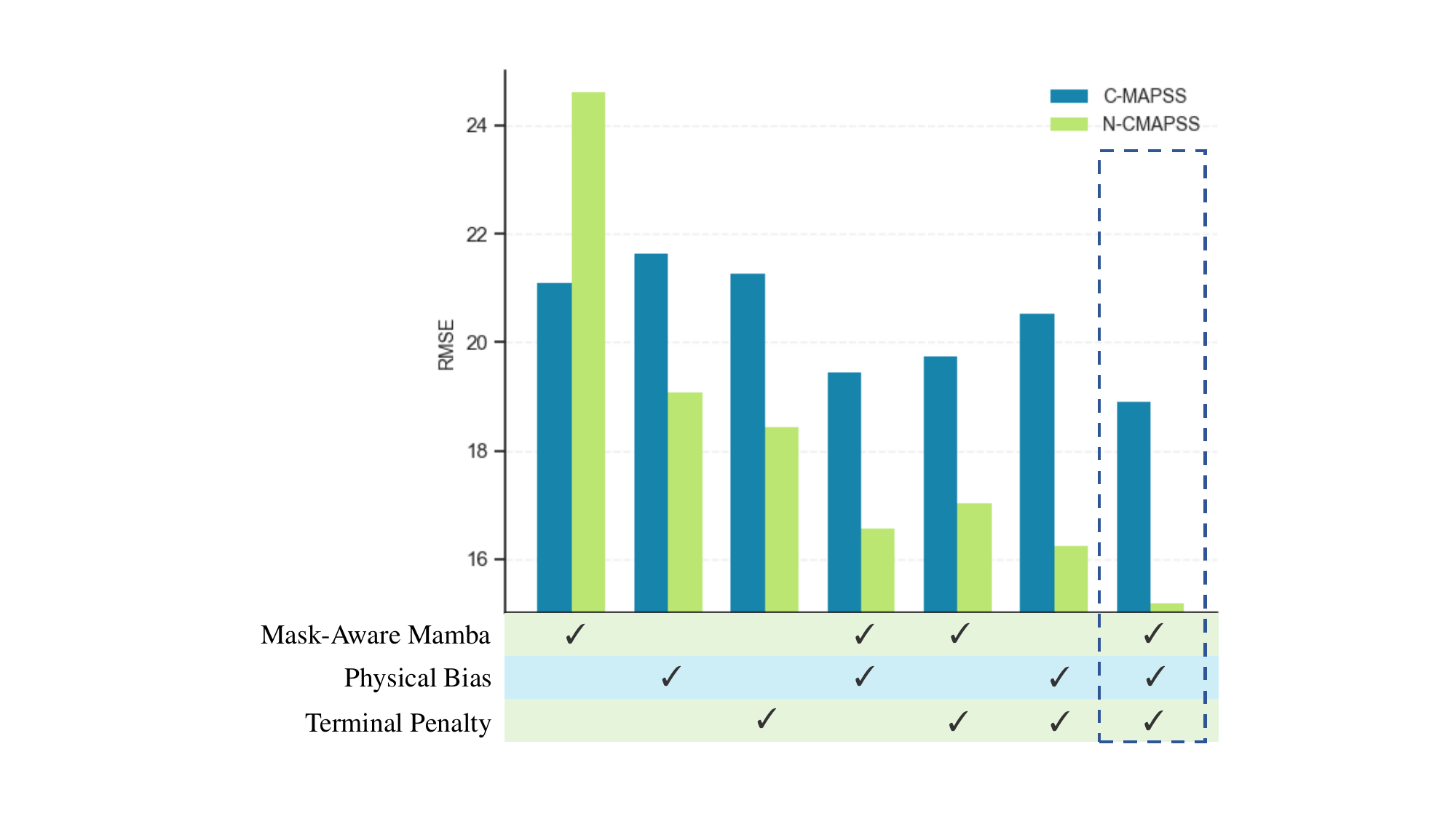}
  \caption{Ablation Study: Impact of Individual Components on RUL Prediction Performance.}
  \label{fig:ablation}
\end{figure}

\subsection{Model Analysis and Visualization}
To rigorously validate that PC-MambaSDE captures intrinsic physical dynamics rather than merely memorizing training patterns, we conduct both microscopic and macroscopic analyses of the learned latent representations, with further extensive experimental analyses provided in \textbf{Appendix~\ref{app:vis}}.
\subsubsection{Microscopic Analysis: SDE Dynamics under Data Sparsity.}
We first examine the microscopic interplay between stochastic diffusion and deterministic physical drift ($b_{phy}$) under varying data availability. As illustrated in Figure \ref{fig:sde_vis}, the model effectively balances data adaptability with physical consistency. In regions with observations, the neural drift adapts to local fluctuations; however, in observation-sparse intervals (indicated by grey regions), $b_{phy}$ acts as a gravitational constraint. This mechanism strictly suppresses unphysical ``self-healing'' hallucinations (upward trends) and enforces monotonic degradation consistent with the irreversibility of damage. Furthermore, the diffusion term naturally quantifies aleatoric uncertainty during these gaps, while the trajectory exhibits robust bidirectional recalibration upon signal restoration---pulling predictions toward the ground truth regardless of whether the physical extrapolation was overly aggressive or conservative.

\subsubsection{Macroscopic Analysis: Manifold Structure and HI Evolution.} 
Complementing the temporal analysis, we investigate the macroscopic quality of the learned representation through the HI and manifold topology. Figure \ref{fig:manifold}(a) demonstrates that the explicitly decoupled HI dimension maintains a stable, monotonic descent that correlates highly with the ground truth, confirming that the Terminal Degradation Penalty effectively filters out high-frequency sensor noise. This physical interpretability is further corroborated by the t-SNE visualization in Figure \ref{fig:manifold}(b), where latent states spontaneously organize into a distinct topological progression. Even without explicit class supervision, the manifold reveals clear separability between Healthy, Degradation, and Critical Failure stages. This structure indicates that PC-MambaSDE extracts discriminative, stage-aware features that adhere to mechanical fatigue laws, validating its robustness for early fault detection.

\begin{figure}[t]
  \centering
  \includegraphics[width=1.0\linewidth, trim=0cm 3cm 0cm 3cm, clip]{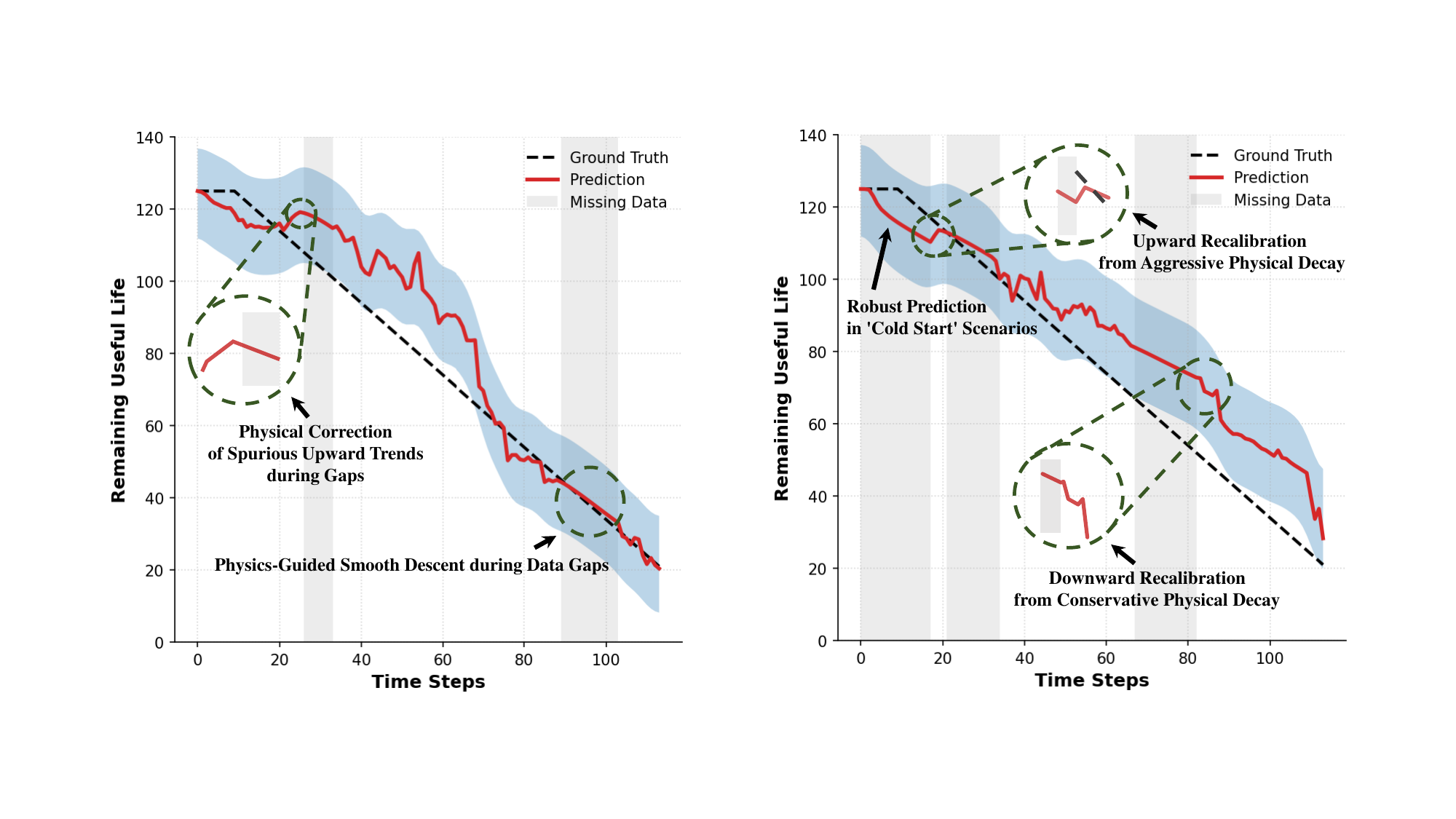}
  \caption{Visualization of Physics-Guided SDE Evolution under Uncertainty.}
  \label{fig:sde_vis}
\end{figure}

\begin{figure}[t]
    \centering  
    \begin{subfigure}[b]{0.18\textwidth}
        \centering
        \includegraphics[width=\linewidth]{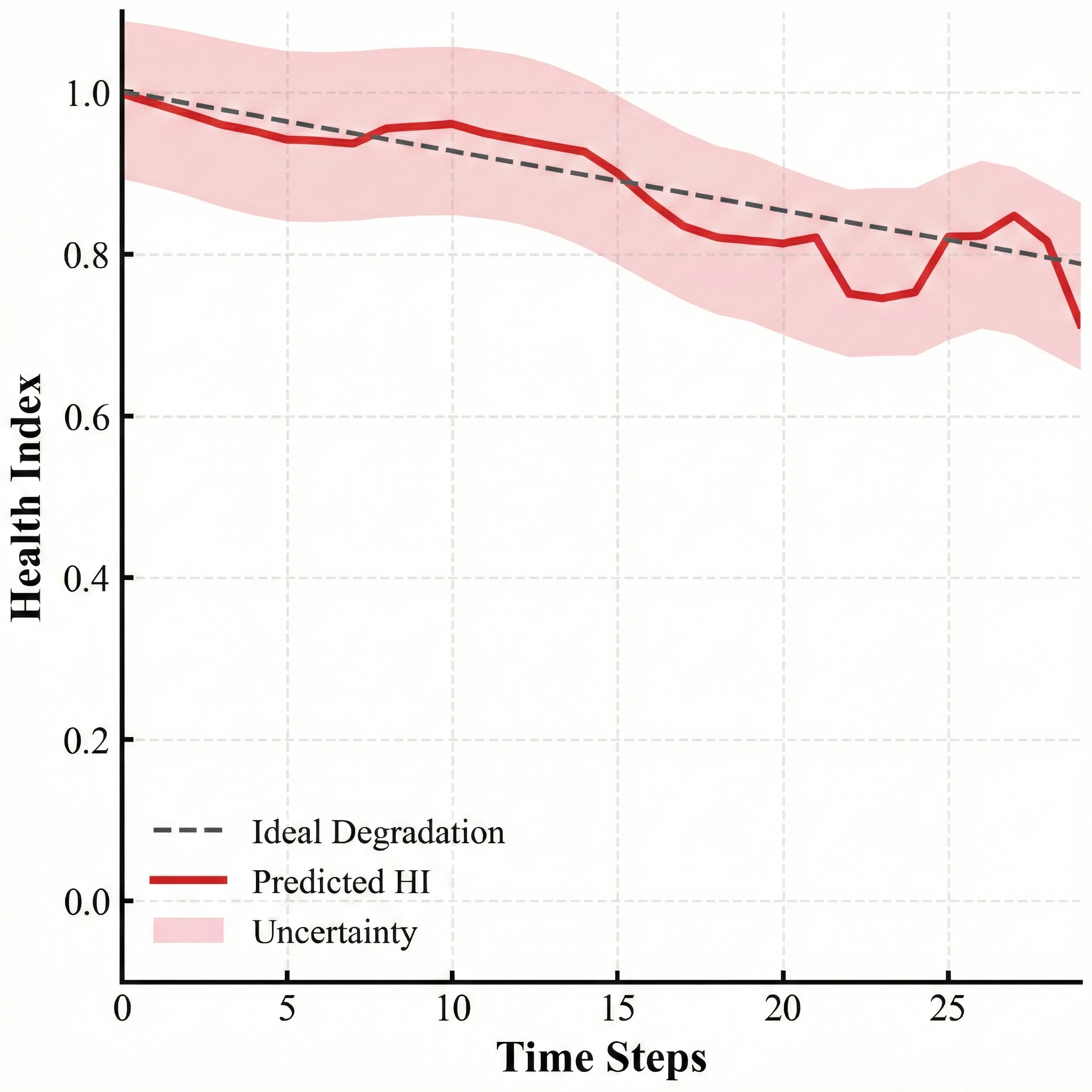}
        \caption{}
        \label{fig:4}
    \end{subfigure}
    \hspace{1.5em} 
    \begin{subfigure}[b]{0.185\textwidth}
        \centering
        \includegraphics[width=\linewidth]{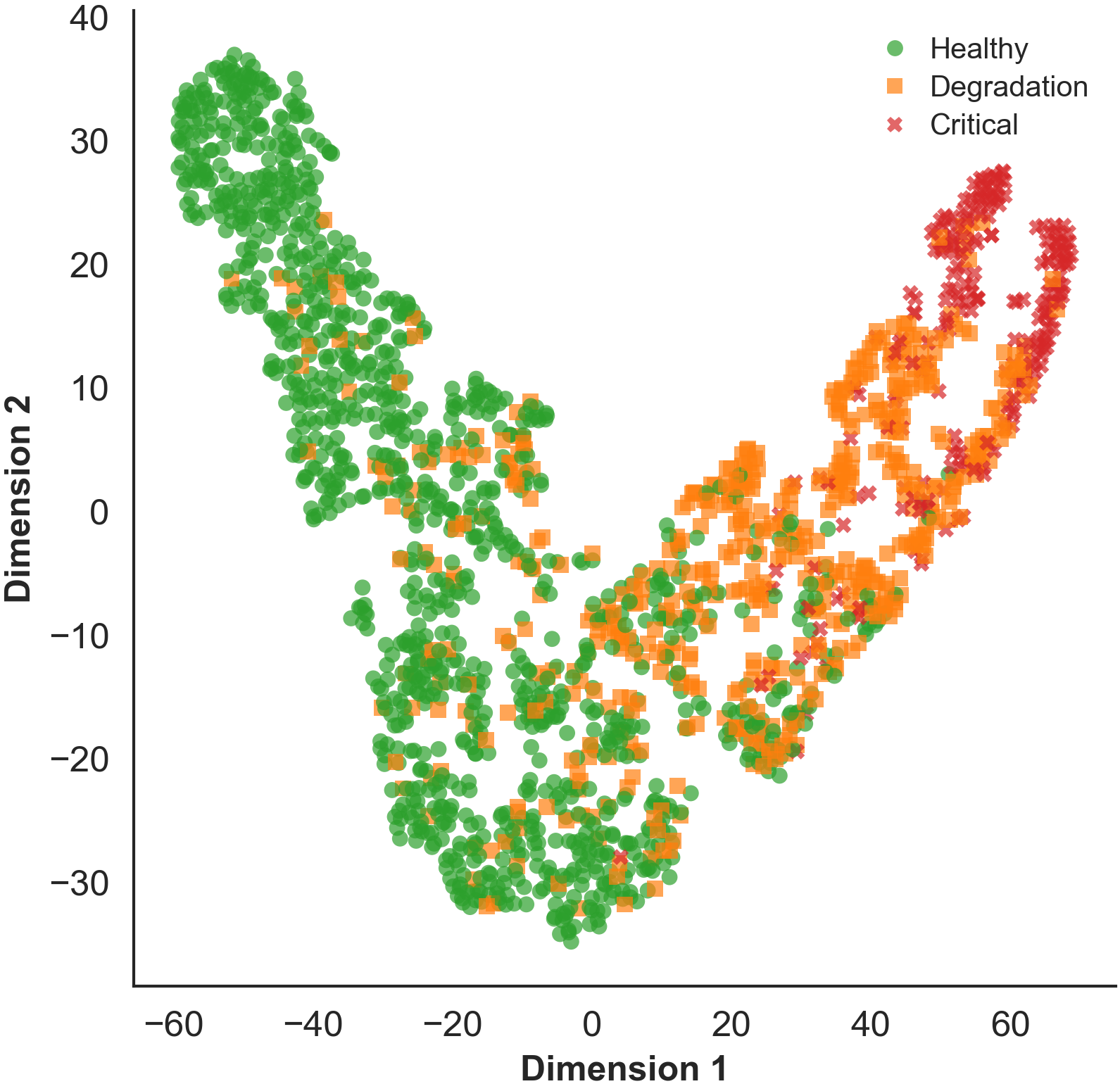}
        \caption{}
        \label{fig:5}
    \end{subfigure}
    
    \caption{Visualization of latent dynamics. (a) Evolution of the Decoupled Health Index. (b) t-SNE visualization of latent space states.}
    \label{fig:manifold}
\end{figure}

\section{Conclusion}
In this paper, we propose PC-MambaSDE, a unified continuous-time framework for RUL prediction under realistic industrial conditions characterized by high sparsity and irregularity. PC-MambaSDE bridges the gap between data-driven representation learning and physical degradation laws through three coupled components: a Mask-Aware Continuous Mamba Encoder that extracts robust control signals from sparse data, a Physics-Guided Latent SDE that superimposes a global physical bias to ensure monotonic degradation trends, and a Terminal Degradation Penalty that guides trajectories toward the failure state. Extensive experiments on C-MAPSS and N-CMAPSS benchmarks, evaluated under a specialized Hybrid Irregularity Generation Scheme, demonstrate that PC-MambaSDE significantly outperforms state-of-the-art discrete and continuous baselines. Notably, our model maintains superior predictive accuracy even in extreme scenarios with up to 90\% data missingness, validating the efficacy of embedding physical priors into continuous-time latent dynamics.

\begin{acks}
This work is supported by the National Natural Science Foun-dation of China under Grant 62136004, Grant 62276130, and Grant 62371234, in part by the National Key R\&D Programof China under Grant 2023YFF1204803, in part by the KeyResearch and Development Plan of Jiangsu Province underGrant BE2022842, ans also by the Ministry of Education, Singapore, under its MOE Tier 1 (SKI 2021\_08\_03).
\end{acks}

\appendix
\bibliographystyle{ACM-Reference-Format}
\balance
\bibliography{KDD_Ref}

\section{Theoretical Analysis and Proofs}
\label{app:proof}

In this section, we provide the detailed mathematical derivations for the PC-MambaSDE framework. We derive the analytical solution to the Physics-Guided Linear SDE, prove the variational bound using Girsanov's Theorem, and demonstrate the stability properties of the Basis-Decomposed Drift.

\subsection{Terminal Degradation Penalty}
\label{app:proof:Penalty}
\begin{lemma}[Monotonicity of Health Index]
Given the rectified drift for the Health Index dimension $z_t^{(h)}$:
\begin{equation}
  dz_t^{(h)} = \left( \mu_{\text{res}}(t, z_t^{(h)}) - |\lambda_{base}| \right) dt + \sigma dW_t
\end{equation}
where $\mu_{\text{res}}(t, z_t^{(h)})$ denotes the residual drift terms (excluding the base degradation rate). The expectation of the degradation rate is strictly negative bounded if the neural perturbation is bounded.
\end{lemma}

\begin{proof}
Taking the expectation of the differential equation:
\begin{equation}
    \frac{d}{dt} \mathbb{E}[z_t^{(h)}] = \mathbb{E}[a_h(t)^T z_t^{(d)}] + \mathbb{E}[[Bu_\phi]_0] - |\lambda_{base}|
\end{equation}
The regularization term $\mathcal{L}_{KL} \propto \|u_\phi\|^2$ pushes $u_\phi \to 0$ in aleatoric regions (missing data). Consequently, as $u_\phi \to 0$, the dominant term becomes $-|\lambda_{base}|$.
\begin{equation}
    \lim_{u_\phi \to 0} \frac{d}{dt} \mathbb{E}[z_t^{(h)}] = -|\lambda_{base}| < 0
\end{equation}
This theoretically guarantees that in the absence of contradictory observational evidence (sparse regions), the model defaults to a monotonic degradation trend, satisfying the ``irreversibility of damage'' physical constraint.
\end{proof}

\subsection{Analytical Solution of the Physics-Guided Linear SDE}

\begin{proposition}[Explicit Solution to Latent Dynamics]
Given the proposed latent SDE defined as:
\begin{equation}
    dZ_t = (A(t)Z_t + B u_\phi(t) + b_{phy})dt + \sigma dW_t
\end{equation}
The explicit solution for the state $Z_t$ given an initial state $Z_{t_0}$ is:
\begin{equation}
    Z_t = \Phi(t, t_0)Z_{t_0} + \int_{t_0}^{t} \Phi(t, s) (B u_\phi(s) + b_{phy}) ds + \int_{t_0}^{t} \Phi(t, s) \sigma dW_s
\end{equation}
where $\Phi(t, s)$ is the fundamental matrix solution to the homogeneous equation $dZ_t = A(t)Z_t dt$.
\end{proposition}

\begin{proof}
Let the drift term be $f(Z_t, t) = A(t)Z_t + \mu(t)$, where $\mu(t) = B u_\phi(t) + b_{phy}$ acts as the external forcing term (neural control + physical bias). We apply the method of integrating factors. Define the integrating factor $M_t$ such that $\frac{d}{dt}M_t = -M_t A(t)$ with $M_{t_0} = I$. The solution is $M_t = \exp(-\int_{t_0}^t A(\tau) d\tau)$.

Applying It\^o's Product Rule to $Y_t = M_t Z_t$:
\begin{equation}
    d(M_t Z_t) = (dM_t)Z_t + M_t(dZ_t) + (dM_t)(dZ_t)
\end{equation}
Since $Z_t$ has finite variation and $W_t$ does not, the cross term $(dM_t)(dZ_t)$ vanishes for the deterministic part of $dM_t$. Substituting the dynamics:
\begin{align}
    d(M_t Z_t) &= (-M_t A(t) dt) Z_t + M_t \left[ (A(t)Z_t + \mu(t)) dt + \sigma dW_t \right] \nonumber \\
    &= -M_t A(t) Z_t dt + M_t A(t) Z_t dt + M_t \mu(t) dt + M_t \sigma dW_t \nonumber \\
    &= M_t \mu(t) dt + M_t \sigma dW_t
\end{align}
Integrating from $t_0$ to $t$:
\begin{equation}
    M_t Z_t - M_{t_0} Z_{t_0} = \int_{t_0}^t M_s \mu(s) ds + \int_{t_0}^t M_s \sigma dW_s
\end{equation}
Multiplying by $M_t^{-1}$ (which corresponds to the transition matrix $\Phi(t, t_0)$) yields the proposition result. In our implementation, we discretize this integral using the matrix exponential for piecewise constant $A(t)$ over intervals $[t_i, t_{i+1})$, allowing for efficient computation via parallel associative scans.
\end{proof}

\section{Additional Experimental Results}
\label{sec:additional_experiments}

\subsection{Computational Efficiency}
\label{app:time}

\begin{table}[htbp]
\centering
\caption{Comparison of Training Time on C-MAPSS Datasets}
\label{app:tab:time}
\resizebox{\linewidth}{!}{
\begin{tabular}{lcccccccc}
\toprule
Method & LSTM+GP & LSTM+MPACE & Sparse MFMLP & Latent-ODE & Latent-SDE & PSR & ACSSM & Ours \\
\midrule
time(s) & 14557 & 7405 & 15832 & 19890 & 20344 & 3621 & 12058 & 9344 \\
\bottomrule
\end{tabular}
}
\end{table}

Table~\ref{app:tab:time} compares training times for convergence on the C-MAPSS dataset. PC-MambaSDE (9,344s) halves the computational cost of continuous-time baselines like Latent-SDE (20,344s) and achieves a 22.5\% reduction versus ACSSM (12,058s). This validates the efficiency of our linear-complexity time-domain Mask-Aware Mamba encoder over quadratic attention mechanisms, confirming our framework offers a scalable solution for industrial monitoring without compromising accuracy.

\subsection{Comparison under PSR Settings}
\label{app:psr}

\begin{table*}[t]
\centering
\caption{Performance comparison on C-MAPSS datasets under different scarcity levels.}
\label{app:tab:scarcity}

\small
\renewcommand{\arraystretch}{1.2} 
\begin{tabularx}{0.9\textwidth}{c l *{4}{>{\centering\arraybackslash}X}}
\toprule
\textbf{Scarcity} & \textbf{Method} & \textbf{FD001} & \textbf{FD002} & \textbf{FD003} & \textbf{FD004} \\
\midrule
\multirow{5}{*}{50\%} 
& LSTM+GP       & 26.45 & 24.76 & 26.43 & 25.26 \\
& LSTM+MPACE    & 17.13 & 18.52 & 19.49 & 20.14 \\
& Sparse MFMLP  & 15.02 & 16.98 & 14.41 & 17.19 \\
& PSR           & 16.74 & 16.59 & 18.91 & 18.57 \\
\rowcolor{gray!15}
& \textbf{PC-MambaSDE (Ours)} & \textbf{12.29 $\pm$ 0.43} & \textbf{11.74 $\pm$ 0.87} & \textbf{13.71 $\pm$ 0.67} & \textbf{15.51 $\pm$ 0.14} \\
\midrule
\multirow{5}{*}{70\%} 
& LSTM+GP       & 29.98 & 30.28 & 30.42 & 28.91 \\
& LSTM+MPACE    & 17.48 & 18.36 & 20.09 & 21.56 \\
& Sparse MFMLP  & 16.85 & 17.08 & 16.06 & 18.18 \\
& PSR           & 16.82 & 16.04 & 18.80 & 18.79 \\
\rowcolor{gray!15}
& \textbf{PC-MambaSDE (Ours)} & \textbf{12.99 $\pm$ 0.51} & \textbf{15.15 $\pm$ 0.58} & \textbf{13.43 $\pm$ 0.34} & \textbf{16.12 $\pm$ 0.78} \\
\midrule
\multirow{5}{*}{90\%} 
& LSTM+GP       & 34.22 & 36.78 & 35.89 & 34.33 \\
& LSTM+MPACE    & 21.85 & 23.10 & 21.10 & 23.54 \\
& Sparse MFMLP  & 20.27 & 18.67 & 19.63 & \textbf{19.03} \\
& PSR           & 19.58 & \textbf{15.89} & 20.45 & 19.92 \\
\rowcolor{gray!15}
& \textbf{PC-MambaSDE (Ours)} & \textbf{16.05 $\pm$ 0.93} & \textbf{18.18 $\pm$ 0.29} & \textbf{17.18 $\pm$ 0.05} & {19.78} $\pm$ 1.2 \\
\bottomrule
\end{tabularx}
\end{table*}

To ensure a fair comparison with the Parameterized Static Regression (PSR) method \cite{ref19}, we evaluated PC-MambaSDE using the specific scarcity and irregularity settings defined in their work. As shown in Table~\ref{app:tab:scarcity}, our method outperforms PSR and Sparse MFMLP across all scarcity levels (50\%, 70\%, 90\%), achieving the lowest RMSE on every subset. This confirms that modeling continuous-time dynamics yields superior results compared to static regression, even when evaluated under protocols originally designed for static models.

\subsection{Microscopic Analysis of Latent Dynamics}
\label{app:vis}

\begin{figure}[t]
  \centering
  \includegraphics[width=1.\linewidth]{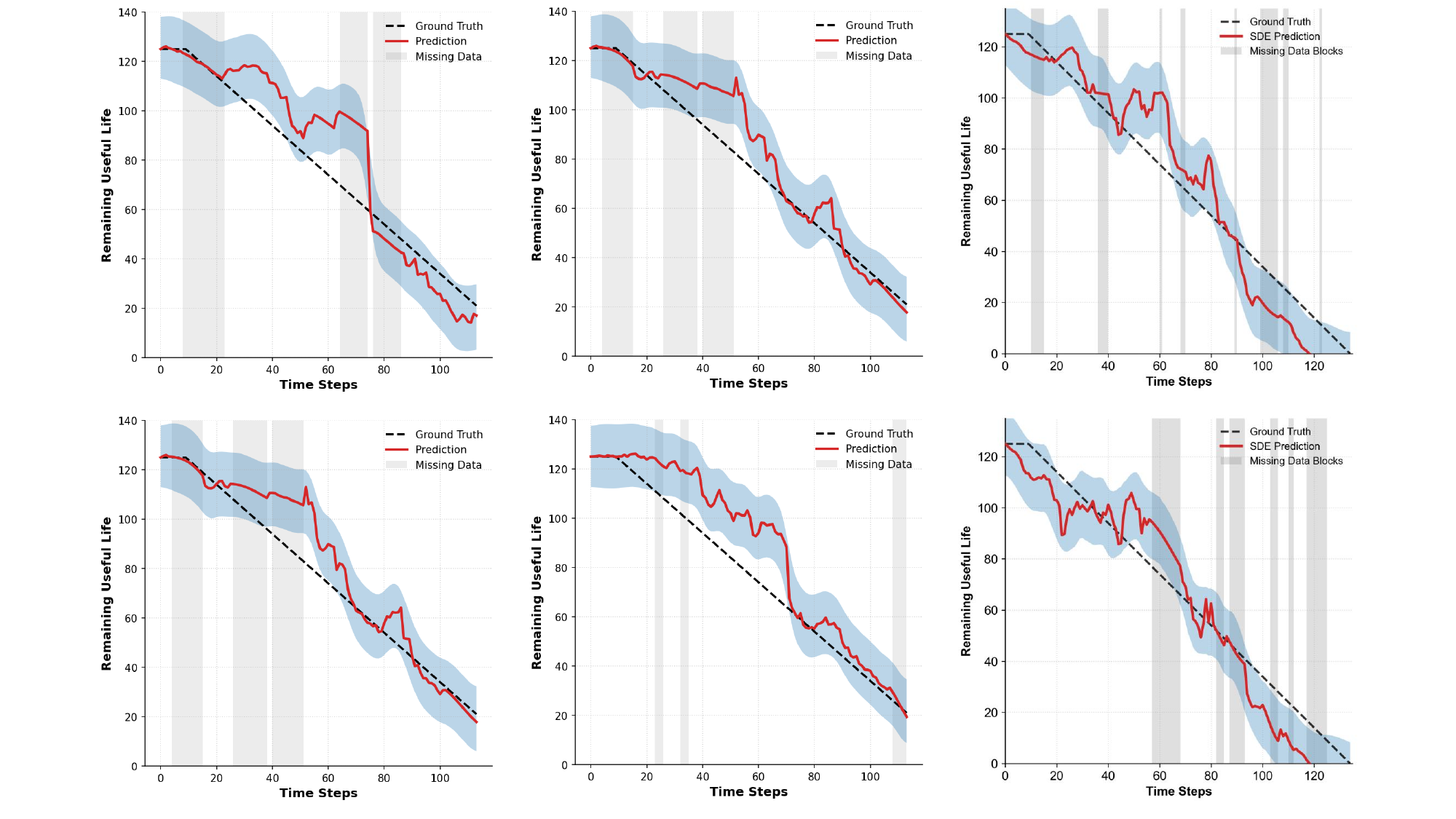}
  \caption{Visualization of Physics-Guided SDE Evolution.}
  \label{app:fig:vis1}
\end{figure}

\begin{figure}[t]
  \centering
  \includegraphics[width=1.0\linewidth, trim=1cm 2cm 0cm 1cm, clip]{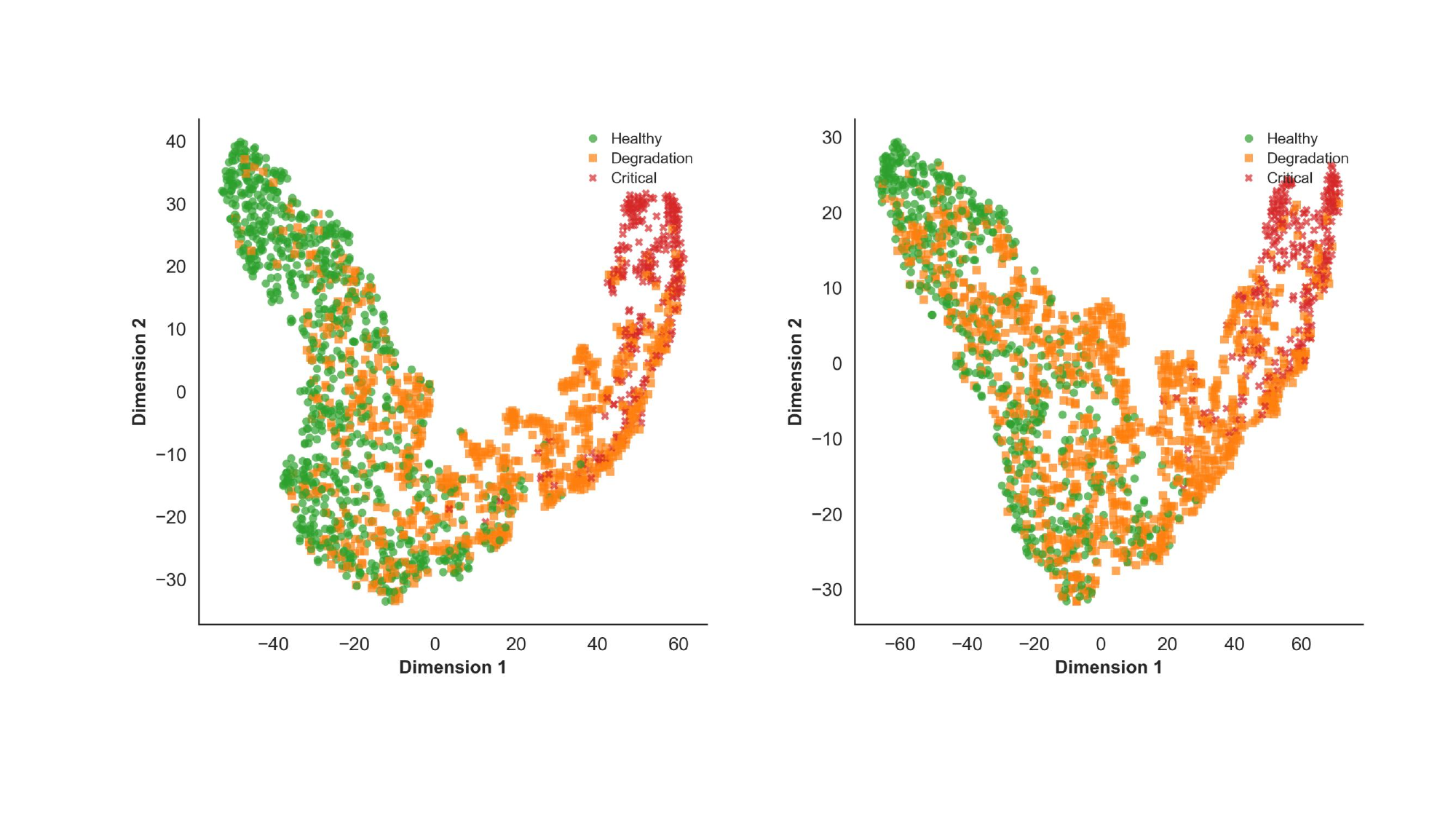}
  \caption{Topological Analysis of the Latent Degradation Manifold.}
  \label{app:fig:vis3}
\end{figure}

Figure~\ref{app:fig:vis1} illustrates how the model leverages the SDE diffusion term to capture uncertainty while employing the Physical Drift to guarantee adherence to physical laws. The grey background bars mark intervals of data missingness ($Mask=0$) where the model relies exclusively on ODE/SDE evolution. A critical observation is the model's behavior when the data-driven prediction erroneously deviates upward from the ground truth—implying non-physical self-healing—just before entering a missing interval; the onset of data missingness triggers an immediate correction driven by the physical bias, causing the trajectory to terminate the upward trend and descend smoothly within the shaded region. The model also demonstrates robust behavior in ``Cold Start'' scenarios where data is missing from the beginning, as the physical bias effectively simulates the degradation trend, and a bidirectional recalibration mechanism is observed upon data restoration: if the physics-guided extrapolation is too aggressive, re-emerging data pulls the prediction upward toward the ground truth, whereas if the physical decay is too conservative, the data drives it further down. This mechanism proves effective even in scattered missing intervals, ensuring consistent performance. Complementing this, Figure~\ref{app:fig:vis3} investigates the sensitivity of the learned representation by visualizing the latent manifold under stricter ``Healthy'' stage definitions (RUL > 100 and RUL > 125). Even under these rigorous thresholds, the topological structure remains distinct, preserving clear separability between Healthy, Degradation, and Critical stages. This persistence of the manifold structure confirms that PC-MambaSDE captures the intrinsic, continuous nature of degradation physics rather than overfitting to specific labeling criteria, thereby validating its robustness for early fault detection tasks.

\section{Dataset Description} \label{app:data}

To rigorously evaluate the proposed framework under realistic industrial conditions, we conduct experiments on two standard predictive maintenance benchmarks: C-MAPSS and N-CMAPSS. Furthermore, to address the challenge of irregular sampling, we subject these datasets to a specialized Hybrid Irregularity Generation Scheme.

\subsection{C-MAPSS Benchmark}
C-MAPSS Benchmark We utilize the canonical Commercial Modular Aero-Propulsion System Simulation (C-MAPSS) benchmark \cite{ref34}, developed by NASA. This dataset simulates the thermodynamic degradation of a commercial turbofan engine under varying fault modes and operating conditions. It comprises four sub-datasets (FD001–FD004):

FD001 \& FD003. Simulate a single operating condition, with FD001 featuring one fault mode and FD003 featuring two.

FD002 \& FD004. Represent complex environments with six distinct operating conditions. FD004 is the most challenging subset, combining multiple fault modes with complex operating conditions.

All engines operate under normal conditions initially and degrade continuously until failure; we strictly adhere to the standard train/test splits provided by the repository.

\subsection{N-CMAPSS Benchmark (Mean-Aggregated)}
To evaluate scalability, we employ the New C-MAPSS (N-CMAPSS) benchmark \cite{ref2} , which provides high-fidelity run-to-failure trajectories reflecting real flight conditions. Unlike the aggregated C-MAPSS, N-CMAPSS records high-frequency sensor snapshots. To handle this effectively, we adopt the preprocessing strategy proposed by Cheng et al. \cite{ref19}. Specifically, we aggregate sensor readings within each flight cycle using a mean-reduction strategy to generate cycle-aggregated trajectories. This transformation preserves the degradation trend while aligning the data scale with standard RUL prediction tasks and significantly reducing computational overhead. The evaluation covers subsets DS01 through DS08c.

\subsection{Hybrid Irregularity Generation Scheme (HIGS)}
We introduce HIGS to transform pristine datasets into irregular sequences via four stochastic perturbation mechanisms:

\textit{Asynchronous Sensor Dropout ($p_{drop}$).} Simulates random packet loss by applying an independent Bernoulli mask to each sensor channel.

\textit{Structural Burst Missingness ($\lambda_{burst}, \mu_{len}$).} Models systematic failures (e.g., network outages) using a Poisson process to generate continuous segments of data loss.

\textit{Temporal Jitter ($\sigma_{jitter}$).} Simulates clock desynchronization errors by perturbing timestamps with Gaussian noise while enforcing monotonicity.

\textit{State-Dependent Noise ($\alpha_{noise}$).} Introduces noise variance that scales dynamically with the degradation level (Health Index), mimicking signal-to-noise ratio decay.

Unless otherwise stated, the ``Base'' setting utilizes $p_{drop}=0.5$, $\lambda_{burst}=0.05$, $\mu_{len}=5$, and $\sigma_{jitter}=0.1$.

\end{document}